\documentclass{article} % For LaTeX2e
\usepackage{iclr2026_conference,times}
\usepackage{MyPreamble}
\usepackage{graphicx}

\newcommand\TODO[1][]{{\color{orange}[TODO\ifthenelse{\equal{#1}{}}{}{: #1}]}}
\newcommand\Ours{GRAPHITE}
\newcommand\OursFull{\emph{\underline{GRA}ph homo\underline{PHI}ly boos\underline{TE}r}}
\newcommand\EQref[1]{Equation~\eqref{#1}}

\usepackage{hyperref}
\usepackage{url}

\title{Graph homophily booster: Reimagining the role of discrete features in heterophilic graph learning}
\author{Ruizhong Qiu\thanks{Equal contribution.},\, Ting-Wei Li\footnotemark[1],\, Gaotang Li, Hanghang Tong\\
University of Illinois Urbana--Champaign, IL, USA \\
$\mathtt{\{rq5\textrm,\,twli\textrm,\,gaotang3\textrm,\,htong\}\,@\,illinois\texttt.edu}$
}

\iclrfinalcopy % Uncomment for camera-ready version, but NOT for submission.
\begin{document}

\maketitle

\begin{abstract}

Graph neural networks (GNNs) have emerged as a powerful approach to modeling graph-structured data and demonstrated remarkable success in many real-world applications such as complex biological network analysis, neuroscientific analysis, and social network analysis. However, existing GNNs often struggle with heterophilic graphs, where connected nodes tend to have dissimilar features or labels. While numerous methods have been proposed to address this challenge, they primarily focus on architectural designs without directly targeting the root cause of the heterophily problem. These approaches still perform even worse than the simplest multi-layer perceptrons (MLPs) on challenging heterophilic datasets. For instance, our experiments show that 23 latest GNNs still fall behind the MLP on the \textsc{Actor} dataset. This critical challenge calls for an innovative approach to addressing graph heterophily beyond architectural designs. To bridge this gap, we propose and study a new and unexplored paradigm: \emph{directly} increasing the graph homophily via a carefully designed graph transformation. In this work, we present a simple yet effective framework called \OursFull{} (\Ours{}) to address graph heterophily. To the best of our knowledge, this work is the first method that explicitly transforms the graph to directly improve the graph homophily. Stemmed from the exact definition of homophily, our proposed \Ours{} creates \emph{feature nodes} to facilitate homophilic message passing between nodes that share similar features.  Furthermore, we both theoretically and empirically show that our proposed \Ours{} significantly increases the homophily of originally heterophilic graphs, with only a slight increase in the graph size. Extensive experiments on challenging datasets demonstrate that our proposed \Ours{} significantly outperforms state-of-the-art methods on heterophilic graphs while achieving comparable accuracy with state-of-the-art methods on homophilic graphs. Furthermore, our proposed graph transformation alone can already enhance the performance of homophilic GNNs on heterophilic graphs, even though they were not originally designed for heterophilic graphs. 
Our code is publicly available at \url{https://github.com/q-rz/ICLR26-GRAPHITE}. 
\end{abstract}
\section{Introduction}

% \RZ{We plan to heavily rewrite the content from the following perspective: (i) Since nodes that share features are more likely to belong to the same class, we can add edges between such node pairs to increase the homophily of the graph. (ii) However, this might add $O(n^2)$ edges. To keep sparsity, we introduce a feature node for each feature and connect it to all nodes that have this feature. This can reduce the number of edges to $O(n\rho)$ where $\rho$ is the sparsity rate of the feature matrix.}

% \RZ{To support our idea, we can provide both empirical results and theoretical analysis as our motivation. We can probably show that our graph transformation improves the homophily.}

Graph neural networks (GNNs) have emerged as a powerful class of models for learning %and making predictions 
on topologically structured data. Their ability to incorporate graph topology and node-level attributes has enabled them to achieve state-of-the-art performance in a wide range of applications. These include protein function prediction, where GNNs model complex biological networks~\citep{you2021deepgraphgo,reau2023deeprank}; neuroscientific analysis, where they are used to model brain networks~\citep{li2023interpretable}; and social network analysis, where they help uncover patterns among users~\citep{li2023survey}.

A critical challenge that many GNNs are faced with is that real-world networks can exhibit heterophily, where connected nodes tend to have dissimilar features or labels. Examples include protein--protein interaction networks where different types of proteins interact~\citep{zhu2020beyond}, or online marketplace networks where buyers connect with sellers rather than other buyers~\citep{pandit2007netprobe}. Standard GNN architectures~\citep{kipf2016semi, wu2019simplifying, velivckovic2017graph, hamilton2017inductive, chen2020simple, abu2019mixhop}, with their heavy reliance on neighborhood aggregation, often struggle with heterophilous graphs since aggregating features from dissimilar neighbors can dilute or distort node representations.
Existing methods for heterophilic graphs mainly focus on designing new GNN architectures as workarounds for heterophilic graphs, such as separating ego and neighbor embeddings \citep{zhu2020beyond,zhang2023steering}, incorporating multi-hop information via learnable weights \citep{chien2020adaptive,dong2024differentiable}, and adaptive self-gating to leverage both low- and high-frequency signals \citep{bo2021beyond}. More recent solutions introduce frequency-based filtering to handle both homophily and heterophily or leverage adaptive residual connections to further enhance flexibility \citep{xu2023node, xu2024slog, yan2024trainable}.

Despite plenty of architectural advances, many GNNs still perform even worse than the simplest multi-layer perceptrons (MLPs) on challenging heterophilic graphs. For instance, Table~\ref{tab:main-res} shows that 23 latest GNNs still fall behind the MLP on the \textsc{Actor} dataset. This critical challenge calls for an innovative approach to addressing graph heterophily beyond architectural designs. 

To bridge this gap, we propose and study a new and unexplored paradigm: \emph{directly} increasing the graph homophily via a carefully designed graph transformation. In this work, we present a simple yet effective framework called \OursFull{} (\Ours{}) to address graph heterophily. To the best of our knowledge, this work is the first method that explicitly transforms the graph to directly improve the graph homophily. 

\begin{figure}[t]
\centering
\includegraphics[width=0.76\linewidth]{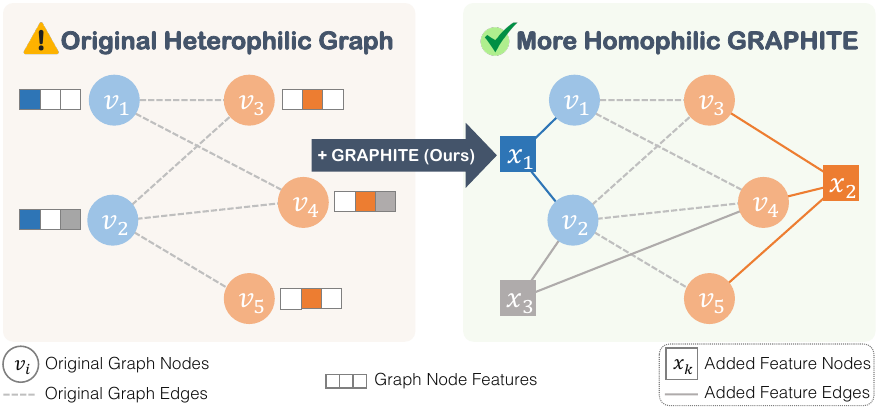}%0.8
\caption{Illustration of our proposed \Ours{}. Feature nodes/edges facilitate more homophilic message passing. For instance, feature node $x_1$ facilitates homophilic message passing between  nodes $v_1,v_2$, and feature node $x_2$ facilitates homophilic message passing among nodes $v_3,v_4,v_5$. 
%\hh{overall, it looks good. 1. arrange the feature nodes to make them close to the corresponding entity nodes, e.g., put x1 between v1 and v2, put x2 in the middle of v3, v4, and v5. 2. add legend: circles for entity nodes; squares for feature nodes; dash line: the original edges; solid line: xxx}\RZ{sounds good. I have updated the figure}
}
\label{fig:illus}
\vspace{-1em}
\end{figure}

Our key idea is rooted in the exact definition of homophily/heterophily. In a homophilic/heterophilic graph, nodes that share similar features are more/less likely to be adjacent, respectively. Therefore, a natural idea to increase the graph homophily is to create ``shortcut" connections between nodes with similar features so as to facilitate homophilic message passing. However, na\"ively adding mutual connections between such node pairs can drastically increase the number of edges. 
% For example, even if a graph has only 2,000 nodes, the na\"ive approach can add as many as 1,999,000 ``shortcut'' edges.
To reduce the number of ``shortcut'' edges, we propose to connect such node pairs \emph{indirectly} instead. In particular, we introduce \emph{feature nodes} as ``hubs'' and connect graph nodes to their corresponding feature nodes. We further theoretically show that our proposed method can provably enhance the homophily of originally heterophilic graphs without increasing the graph size much. 

Our main contributions are summarized as follows:
\begin{itemize}
    \item\textbf{New paradigm.} We propose and study a new and unexplored paradigm: \emph{provably} increasing the graph homophily directly via non-learning-based graph transformation. This paper is the first work on this paradigm to the best of our knowledge. 
    \item\textbf{Proposed method.} We propose a simple yet effective method called \Ours{}, which creates feature nodes as ``shortcuts'' to facilitate homophilic message passing between nodes with similar features. 
    \item\textbf{Theoretical guarantees.} We theoretically show that  \Ours{} can \emph{provably} enhance the homophily of originally heterophilic graphs with only a \emph{slight} increase in size. 
    \item\textbf{Empirical performance.} Extensive experiments on challenging datasets demonstrate the effectiveness of our proposed \Ours{}. \Ours{} \emph{significantly} outperforms state-of-the-art methods on heterophilic graphs while achieving \emph{comparable} accuracy with state-of-the-art methods on homophilic graphs. Furthermore, our proposed graph transformation alone can already enhance the performance of homophilic GNNs on heterophilic graphs. 
\end{itemize}

% Extensive experiments on challenging benchmark datasets demonstrate the effectiveness of our approach. On the standard heterophily benchmarks including Actor, Squirrel-Filtered, Chemeleon-Filtered and Minesweeper datasets, \Ours{} consistently outperforms state-of-the-art methods by significant margins. At the same time, \Ours{} achieves comparable or better results on homophilious benchmarks, demonstrating its versatility across different graph structures. The success of \Ours{} across diverse datasets and heterophily levels demonstrates that our simple architectural \hh{we are not changing architecture, but the graph right? check}modification effectively addresses the heterophily challenge while maintaining computational efficiency. The source code and implementation of our method will be made publicly available upon publication.
% \hh{mention/promise the code release}\tw{added}
\section{Preliminaries}
\label{sec:prob-def}

% In this paper, we consider graph datasets with discrete features, i.e., the node attributes are represented as binary or categorical vectors, such as bag-of-words features in Cora and Citeseer datasets~\cite{sen2008collective}. Formally, we define graphs with discrete feature sets and describe the problem statement in this section. 

\subsection{Notation}
% \TODO[graph, gnn, homophily]
% \tw{I use $C$ to denote the number of classes and use $y_u$ to denote the node label of node $u$ in the homophily analysis section}
% Heterogeneous feature sets in graphs pose unique challenges in representation and learning, as they deviate from the standard assumption of uniform feature dimensions across all nodes within the graph. Such graphs frequently occur in real-world applications where entities represented by nodes have varying types and quantities of information. For example, on social networks, some user profiles may contain detailed demographic and activity data, while others may only have sparse information due to privacy restrictions. Similarly, in knowledge graphs, entities may be associated with attributes of varying cardinality and relevance, depending on their roles within the graph.

% To capture this variability formally, we define a graph with discrete feature sets as follows.

An undirected graph with discrete node features can be represented as a triple $\mathcal{G} = ( \mathcal{V}, \mathcal{E}, \BM X )$, where $\mathcal{V}=\{v_1,\dots,v_{|\CAL V|}\}$ denotes the node set, $\mathcal{E}\subseteq\CAL V\times\CAL V$ denotes the edge set, $\BM X \in \{0,1\}^{\mathcal{V}\times \mathcal{X}}$ is a binary node feature matrix representing discrete node features, and $\mathcal{X}=\{1,\dots,|\CAL X|\}$ is the feature set containing all the discrete node features. In addition to that, each graph node $v_i \in \mathcal{V}$ has a node label $y_{v_i} \in\mathcal{Y}$, where $\mathcal{Y}$ is the label set with $C = |\mathcal{Y}|$ classes.

% \begin{definition}[Graph with discrete feature set] A graph is defined as $\mathcal{G} = ( \mathcal{V}, \mathcal{E}, \mathbf{X} )$, where $\mathcal{V}$ is the node set, $\mathcal{E}$ is the edge set and $\mathbf{X}$ is the feature collection = $[ \mathbf{x}_1, \mathbf{x}_2, \cdots, \mathbf{x}_{|\mathcal{V}|}]$ where $\mathbf{x}_i \in \mathbb{R}^{|F_i|}, \forall i \in [|\mathcal{V}|]$. Note that  $F_i$ is the available set of features for node $i$.
% \label{def-1}
% \end{definition}

% \hh{do we still need to emphasize this (i.e., the irregularity or heterogeneity of the feature). i think what matters is the sparsity of the feature right?}\RZ{I will rewrite preliminaries soon}

% This definition encapsulates the irregularity of node features within the graph. Unlike conventional graph representations, where each node is assumed to have a fixed-dimensional feature vector, this formulation allows nodes to have variable-length feature vectors, determined by $|F_i|$. 
% The heterogeneity of these feature sets presents unique challenges for graph learning models, as they must effectively handle nodes with differing feature spaces while preserving the relationships encoded in the graph topology.

\subsection{Problem Definition}

In this paper, we study two key problems: (i) how to transform a graph to increase its homophily and (ii) how to perform node classification on a heterophilic graph datasets. Formally, we introduce the problem definitions as follows.

\begin{PRB}[boosting graph homophily] Given a highly heterophilic graph, transform the graph to increase its homophily. \textbf{Input:} a heterophilic graph $\mathcal{G}$. \textbf{Output:} a transformed graph $\mathcal{G}^*$  with higher homophily.
\label{def:homo-boost}
\end{PRB}

% Building on Definition~\ref{def-1}, the central problem we address in this work can be expressed as follows.\hh{in Section 2 and Sections 3.1-3.3, consider to re-use the example in Figure 1 to explain some key notations}

\begin{PRB}[semi-supervised node classification on a heterophilic graph]
Given a heterophilic graph and a set of labelled nodes, train a model to predict the labels of unlabelled nodes. \textbf{Input:} (i) a heterophilic graph \( \mathcal{G}=(\CAL V,\CAL E,\BM X)\); (ii) a labelled node set \( \mathcal{V}_\textnormal{L} \subset \mathcal{V} \) whose node labels $[y_{v_i}]_{v_i\in\CAL V_\textnormal L}$ are available. \textbf{Output:} the predicted labels of unlabeled nodes \( \mathcal{V} \setminus \mathcal{V}_\textnormal{L} \).

\label{def:node-classify}
\end{PRB}

% Given a graph defined in Definition~\ref{def-1}, we aim to find a mapping function $\Theta: \mathcal{V} \rightarrow \mathbb{R}^\ell$ that projects each node to an $\ell$-dimensional space. The resulting $\ell$-dimensional node embeddings are used to solve downstream node classification tasks. 

% The goal of Problem~\ref{def-2} is to generate node embeddings that effectively capture both the structural information from the graph and the discrete nature of node features. Achieving this requires designing models and algorithms that address the challenges introduced by heterogeneous feature sets, such as handling variable feature dimensions, ensuring robustness to missing data, and maintaining computational efficiency. The solutions to this problem have far-reaching implications for applications across multiple domains, including social network analysis, recommendation systems, and biological network modeling.

\section{Proposed Method: \Ours{}}
In this section, we propose a simple yet effective graph transformation method called \OursFull{} (\Ours{}) that can efficiently increase the homophily of a graph. In Section~\ref{ssec:motiv}, we will introduce the motivation of our proposed \Ours{}. %In Section~\ref{ssec:method},
First, we will present the design of our proposed method \Ours{}. %In Section~\ref{ssec:method-detail},
Then, we will describe the neural architecture of our proposed method. Due to the page limit, proofs of theoretical results are deferred to the appendix.%Appendix~\ref{app:proofs}. 
%\RZ{I've rewritten the method section}

\subsection{Motivation: Na\"ive Homophily Booster}\label{ssec:motiv}

Graph heterophily is a ubiquitous challenge in graph-based machine learning. On a highly heterophilic graph, many neighboring nodes exhibit dissimilar features or belong to different classes. As a result, graph heterophily limits the effectiveness of GNN message passing, as standard aggregation schemes might fail to capture meaningful patterns in heterophilic neighbors.

Existing methods for heterophilic graphs mainly focus on designing workarounds such as new architectures or learning paradigms for heterophilic graphs, including adaptive message passing, higher-order neighborhoods, or alternative propagation mechanisms that leverage both local and global graph structures. 

In contrast to existing workaround methods, we propose a new method that aims to directly increase the homophily of the graph via a specially designed graph transformation. To the best of our knowledge, this work is the first method that explicitly transforms the graph to improve the homophily of the graph. 

Our idea is rooted in the exact definition of homophily and heterophily. In a heterophilic graph, nodes that share similar features are more likely to be non-adjacent. However, in a homophilic graph, nodes that share similar features should be more likely to be neighbors. Therefore, a natural idea to increase the homophily of the graph is to create ``shortcut" connections between nodes with similar features, which will facilitate homophilic message passing between them.
%\hh{which will in turn facilitate the message passing between them}\RZ{sounds good. added}

Before we introduce the proposed method, let's consider the following na\"ive approach to implementing the aforementioned idea: For each pair of nodes $v_i,v_j \in\CAL V$, if they share at least a feature (i.e., $\|\BM X[v_i,:]\land\BM X[v_j,:]\|_\infty>0$), we add a ``shortcut'' edge $(v_i,v_j)$ between them. Let's call this approach the \emph{na\"ive homophily booster} (NHB). The following Theorem~\ref{thm:naive} shows that NHB can indeed increase the homophily of the graph under mild and realistic assumptions. 

\begin{THM}[na\"ive homophily booster]\label{thm:naive}
Given a heterophilic graph $\CAL G=(\CAL V,\CAL E,\BM X)$, let $\CAL E^\dagger$ denote the set of edges after adding the NHB ``shortcut'' edges, and let $\CAL G^\dagger:=(\CAL V,\CAL E^\dagger,\BM X)$ denote the graph transformed by NHB. Under mild and realistic assumptions in Appendix~\ref{app:ass}, we have
\GA{
\OP{hom}(\CAL G^\dagger)>\OP{hom}(\CAL G),\\
|\CAL E^\dagger|-|\CAL E|\le O(|\CAL V|^2).\label{eq:naive-edges}
}
\end{THM}

However, \EQref{eq:naive-edges} also shows that NHB is extremely inefficient despite its effectiveness in increasing homophily. For instance, even if the graph has only 2,000 nodes, NHB can add as many as 1,999,000 ``shortcut'' edges. The plenty of ``shortcut'' edges can drastically slow down the training and the inference process of GNNs. Hence, this na\"ive approach is computationally impractical for GNNs. To address this computational challenge, we will instead propose an efficient homophily booster %in Section~\ref{ssec:method}
via a more careful design of ``shortcut'' edges. 

\begin{figure*}[t]
\centering
\includegraphics[width=0.9\linewidth]{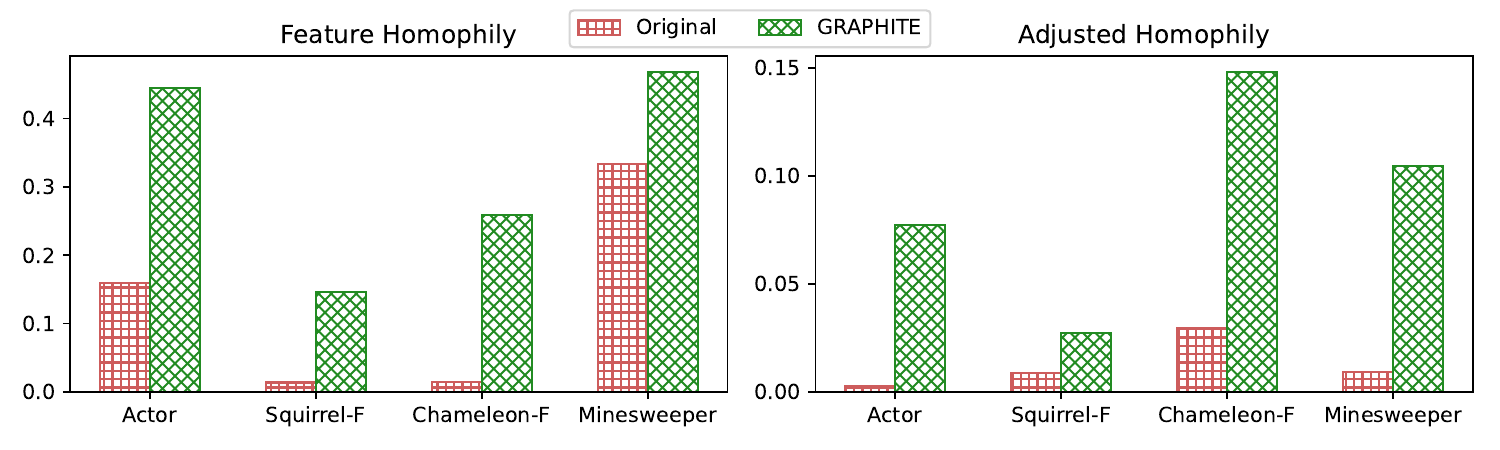}
\caption{Our proposed \Ours{} significantly increases the homophily of originally heterophilic graphs. We report two latest homophily metrics: \emph{feature homophily} \citep{jin2022raw} and \emph{adjusted homophily} \citep{platonov2024characterizing}. 
}
\label{fig:homo-analysis}
\vspace{-1em}
\end{figure*}

\subsection{Efficient Graph Homophily Booster}\label{ssec:method}

To address the computational inefficiency of the motivating na\"ive approach above, %in Section~\ref{ssec:motiv},
we propose an efficient, simple yet effective graph transformation method called \OursFull{} (\Ours) in this subsection. 

Note that the large number of NHB ``shortcut'' edges is because NHB \emph{directly} connects nodes with similar features. Since there are $O(|\CAL V|^2)$ node pairs in a graph, then the total number of added NHB ``shortcut'' edges can be as large as $O(|\CAL V|^2)$. 

To reduce the number of ``shortcut'' edges, we propose to connect such node pairs \emph{indirectly} instead. In particular, if we can create a few auxiliary ``hub'' nodes so that all such node pairs are \emph{indirectly} connected through the ``hub'' nodes, then we will be able to significantly reduce the number of ``shortcut'' edges at only a small price of adding a few ``hub'' nodes. Therefore, we need to develop an appropriate design of the ``hub'' nodes. 

\vspace{0.5em}\noindent\textbf{Graph transformation.} Following the aforementioned motivation, we propose to create a \emph{feature node} $x_k$ for each feature $k$ to serve as the ``hub'' nodes. Let $\CAL V_\CAL X$ denote the set of feature nodes:
\AL{\CAL V_\CAL X:=\{x_k:k\in\CAL X\}.}
To distinguish feature nodes $\CAL V_\CAL X$ from nodes $\CAL V$ in the original graph, we call $\CAL V$ \emph{graph nodes} from now on. For each graph node $v_i \in \CAL V$, if graph node $v_i$ has feature $k$ (i.e., $\BM X[v_i,k]=1$), we add an edge $(v_i,x_k)$ to connect the graph node $v_i$ and the feature node $x_k\in\CAL V_\CAL X$, and we call it a \emph{feature edge}. Let $\CAL E_\CAL X$ denote the set of feature edges:
\AL{\CAL E_\CAL X:={}&\{(v_i,x_k):v_i \in\CAL V,\,x_k \in\CAL V_\CAL X,\,\BM X[v_i,k]=1\}\nonumber\subseteq{} \CAL V\times\CAL V_\CAL X.}
To distinguish feature edges $\CAL E_\CAL X$ from the original edges $\CAL E$, we call $\CAL E$ \emph{graph edges} from now on.

Finally, we define the transformed graph $\CAL G^*=(\CAL V^*,\CAL E^*,\BM X^*)$ as follows. The nodes $\CAL V^*$ of the transformed graph $\CAL G^*$ are the original graph nodes $\CAL V$ and the added feature nodes $\CAL V_\CAL X$:
$\CAL V^*:=\CAL V\cup\CAL V_\CAL X$.
The edges $\CAL E^*$ of the transformed graph $\CAL G^*$ are the original graph edges $\CAL E$ and the added feature edges $\CAL E_\CAL X$:
$\CAL E^*:=\CAL E\cup\CAL E_\CAL X$.
%\hh{we might consider to use a $2 \times 2$ block matrix: the first diagonal block is A, the two off-diagonal blocks are X and X', and the 2nd diagonal block is S.}\RZ{good idea. added}
We can also equivalently define the edges of the transformed graph $\CAL G^*$ by its adjacency matrix. Let $\BM A$ denote the adjacency matrix of the original graph $\CAL G$. Then, the adjacency matrix $\BM A^*$ of the transformed graph $\CAL G^*$ can be expressed in a block matrix form:
\AL{\BM A^*=\begin{bmatrix}
\BM A&\BM X\\
\BM X\Tp&\BM0
\end{bmatrix}.}

It remains to define node features $\BM X^*\in\BB R^{\CAL V^*\times\CAL X}$ of the transformed graph. For each graph node $v_i\in\CAL V$, we use %either zeros \hh{why zeros?}\RZ{Actually using zeros is better on only one dataset. Maybe we can omit it here and mention it as an implementation detail?}\hh{sounds good} or 
its original features as its node features:$\BM X^*[v_i,:]:=\BM X[v_i,:]$.
For each feature node $x_k\in\CAL V_\CAL X$, we define its node feature as the average feature vector among the graph nodes $v_i$ that are connected to feature node $x_k$:
\AL{\BM X^*[x_k,:]:=\frac{1}{|\CAL E_\CAL X\cap(\CAL V\times\{x_k\})|}\sum_{\begin{subarray}{c}v_i:(v_i,x_k)\in\CAL E_\CAL X\end{subarray}}\BM X[v_i,:].}

Our proposed graph transformation \Ours{} is illustrated in Figure~\ref{fig:illus}. In this example, $\{v_1,v_2,v_3,v_4,v_5\}$ are the graph nodes, where $v_1,v_2$ belong to one class, and $v_3,v_4,v_5$ belong to the other class. Our proposed \Ours{} adds feature nodes $x_1,x_2,x_3$ to the graph. We can see that feature node $x_1$ facilitates homophilic message passing between $v_1,v_2$, and that feature node $x_2$ facilitates homophilic message passing among $v_3,v_4,v_5$. 

\vspace{0.5em}\noindent\textbf{Theoretical guarantees.} The transformed graph $\CAL G^*$ enjoys a few theoretical guarantees. First, an important property of the feature edges is that every pair of nodes that share features can be connected through feature edges within two hops, as formally stated in Observation~\ref{lem:2hop}. This ensures that nodes with similar features are close to each other on the transformed graph $\CAL G^*$, facilitating homophilic message passing. 
\begin{OBS}[two-hop indirect connection]\label{lem:2hop}
For each pair of nodes $u,v\in\CAL V$, if they share at least a feature (i.e., $\|\BM X[v_i,:]\land\BM X[v_j,:]\|_\infty>0$), then $v_i$ and $v_j$ are two-hop neighbors on the transformed graph $\CAL G^*$. 
\end{OBS}

Furthermore, we theoretically show that our proposed graph transformation \Ours{} can increase the homophily of the graph without increasing the size of the graph much, as formally stated in Theorem~\ref{thm:eff}. 

\begin{THM}[efficient homophily booster]\label{thm:eff}
Given a heterophilic graph $\CAL G=(\CAL V,\CAL E,\BM X)$, let $\CAL G^*:=(\CAL V^*,\CAL E^*,\BM X^*)$ denote the graph transformed by our proposed \Ours{}. Under mild and realistic assumptions in Appendix~\ref{app:ass}, we have
\GA{
\OP{hom}(\CAL G^*)>\OP{hom}(\CAL G),\\
|\CAL V^*|\le O(|\CAL V|),\quad
|\CAL E^*|\le O(|\CAL E|).\label{eq:feat-edges}
}
\end{THM}
The effectiveness of our proposed \Ours{} is also empirically validated in Section~\ref{ssec:exp:homo}. As shown in Figure~\ref{fig:homo-analysis}, our proposed GRAPHITE significantly increases the homophily of originally heterophilic graph. 

% \vspace{0.5em}\noindent\textbf{A motivating na\"ive approach.} \TODO

% \vspace{0.5em}\noindent\textbf{\TODO.} \TODO

\subsection{Neural Architecture}\label{ssec:method-detail}
The transformed graph $\CAL G^*$ can be readily fed into existing GNNs to boost their performance, even when the GNNs were originally designed for homophilic graphs, as demonstrated in Table~\ref{tab:abla}.  Meanwhile, to maximize the GNN performance on the transformed graph $\CAL G^*$, we introduce a GNN architecture specially designed for the transformed graph in this subsection.

% \hh{1. we might consider to say sth like, the transformed graph can be fed into a variety of the existing GNN models to boost their classification performance, even if they were originally designed for the homophilic graphs. In this paper, we use a modified FAGNN as an example. 2. *if we have time*, we can add an overall algorithm (or algorithm sketch)}\tw{we can say FAGCN is just an example and the graph transformation part is the main contribution}\RZ{sounds good. I will experiment with other GNNs to see if our graph transformation helps other GNNs as well}\hh{let's save the experiments *after* the ddl}

% In this subsection, we introduce the GNN architecture for the transformed graph $\CAL G^*$.

%\vspace{0.5em}\noindent\textbf{Edge weights.}
To help the GNN distinguish graph nodes $\CAL V$ from feature nodes $\CAL V_\CAL X$, we use different edge weights for different edges. As a reference weight, suppose that graph edges $\CAL E$ have weight $w_\CAL E:=1$. Let 
% \hh{is $w_\CAL X$ a parameter that we need to tune?}\RZ{yes. different datasets need different w to achieve the best results}
$w_\CAL X>0$ denote the weight of feature edges $\CAL E_\CAL X$. Following GCN \citep{kipf2016semi}, we also use self-loops in GNN message passing; let $w_0>0$ denote the weight of self-loops. 

Let $d_u$ denote the weighted degree of each node $u\in\CAL V^*$. Specifically, for each graph node $v_i \in\CAL V$,
\AL{d_{v_i}:=w_0+\sum_{(v_i,v_j)\in\CAL E}w_{\CAL E}+\sum_{(v_i,x_k)\in\CAL E_\CAL X}w_\CAL X;}
and for each feature node $x_k\in\CAL V_{\CAL X}$,
\AL{d_{x_k}:=w_0+\sum_{(v_i,x_k)\in\CAL E_\CAL X}w_\CAL X.}

Inspired by FAGCN \citep{bo2021beyond}, we use a self-gating mechanism in GNN aggregation. For each node $u\in\CAL V^*$, let $\BM h_u\in\BB R^m$ denote the embedding of node $u$ before GNN aggregation, where $m$ is the embedding dimensionality. Then, the self-gating score $\alpha_{u,u'}$ between two nodes $u,u'\in\CAL V^*$ is defined as
\AL{\alpha_{u,u'}:=\tanh\Big(\frac{\BM a\Tp(\BM h_u\mathbin{\|}\BM h_{u'})+b}{\tau}\Big).}
where $\|$ denotes the concatenation operation, $\BM a\in\BB R^{2m}$ and $b\in\BB R$ are learnable parameters, and $\tau>0$ is a temperature hyperparameter. 

Next, we describe our aggregation mechanism. For each node $u\in\CAL V^*$, let $\BM h'_u\in\BB R^m$ denote the embedding of node $u$ after GNN aggregation. For each graph node $v_i \in\CAL V$, we define
\AL{\BM h'_{v_i}:={}&\frac{w_0\alpha_{v_i,v_i}}{\sqrt{d_{v_i}}\sqrt{d_{v_i}}}\BM h_{v_i}+\!\!\!\!\!\!\sum_{(v_i,v_j)\in\CAL E}\!\!\frac{\alpha_{v_i,v_j}}{\sqrt{d_{v_i}}\sqrt{d_{v_j}}}\BM h_{v_j}%\nonumber\\&
+\!\!\!\!\!\!\sum_{(v_i,x_j)\in\CAL E_\CAL X}\!\!\!\frac{w_\CAL X\alpha_{v_i,x_k}}{\sqrt{d_{v_i}}\sqrt{d_{x_k}}}\BM h_{x_k};}
and for each feature node $x_k\in\CAL V_\CAL X$, we define
\AL{\BM h'_{x_k}:=\frac{w_0\alpha_{x_k,x_k}}{\sqrt{d_{x_k}}\sqrt{d_{x_k}}}\BM h_{x_k}+\sum_{(v_i,x_k)\in\CAL E_\CAL X}\frac{w_\CAL X\alpha_{v_i,x_k}}{\sqrt{d_{v_i}}\sqrt{d_{x_k}}}\BM h_{v_i}.}
Furthermore, we add a multi-layer perceptron (MLP) with residual connections after each GNN aggregation. We use the GELU activation function \citep{hendrycks2016gaussian}. 

\begin{table*}[t]%{\linewidth}
\centering
\caption{Summary of dataset statistics. We use four heterophilic graphs and two homophilic graphs.}
\label{tab:data}
\resizebox{0.9\linewidth}!{\begin{tabular}{l|cccc|cc}
\specialrule{3\arrayrulewidth}{5\arrayrulewidth}{5\arrayrulewidth}
\multirow{2.2}*{\textbf{Statistic}}&\multicolumn{4}{c|}{\textbf{Heterophilic Graphs}}&\multicolumn{2}{c}{\textbf{Homophilic Graphs}}\\
&\textsc{Actor}&\textsc{Squirrel-F}&\textsc{Chameleon-F}&\textsc{Minesweeper}&\textsc{Cora}&\textsc{CiteSeer}\\
\midrule
\# Nodes&7600&2223&890&10000&2708&3327\\
\# Edges&33544&46998&8854&39402&5429&4732\\
\# Features&931&2089&2325&7&1433&3703\\
\# Classes&5&5&5&2&7&6\\
%Edge Hom.&&&&0.68&0.81&0.74\\
Homophily & 0.0028 & 0.0086 & 0.0295 & 0.0094 & 0.7711 & 0.6707\\
\specialrule{2.4\arrayrulewidth}{5\arrayrulewidth}{5\arrayrulewidth}
\end{tabular}}
\vspace{-1em}
\end{table*}
% \end{wraptable}

\section{Experiments}
%\RZ{TODO: Research questions}

We conduct extensive experiments on both heterophilic and homophilic datasets to answer the following research questions:
\begin{enumerate}
\renewcommand\labelenumi{\textbf{RQ\theenumi:}\!\!\!\!\!\!\!\!\!}
\item\label{rq:main}\,\,\,\,\,\,\,\,\,How does our proposed \Ours{} compare with state-of-the-art methods? 
\item\label{rq:homo}\,\,\,\,\,\,\,\,\,How much improvement can our \Ours{} achieve in terms of the graph homophily?
\item\label{rq:abla}\,\,\,\,\,\,\,\,\,Can our graph transformation alone enhance the accuracy of homophilic GNNs?
\item\label{rq:feat}\,\,\,\,\,\,\,\,\,How should we design the features of feature nodes in our \Ours{}?
\item\label{rq:eff}\,\,\,\,\,\,\,\,\,How efficient is our \Ours{}?
\item\label{rq:hyperparams}\,\,\,\,\,\,\,\,\,How effective is our \Ours{} under various hyperparameters?
\end{enumerate}

\subsection{Experimental Settings}

\vspace{0.5em}\noindent\textbf{Datasets.}
We evaluate \Ours{} and various baseline methods across six real-world datasets. The dataset statistics are summarized in  Table~\ref{tab:data}. The reported homophily is the \textit{adjusted homophily} introduced in~\citet{platonov2024characterizing}, which exhibits more desirable properties compared to traditional edge/node homophily. We leverage \textit{adjusted homophily} to categorize the datasets into two groups: \textit{heterophilic} and \textit{homophilic}. Please see Appendix~\ref{app:data} for dataset descriptions.
% We specifically include these datasets to test our method's robustness across different types of graph structures and node-label relationships.

%In our experiments, we first select GraphSAGE~\cite{hamilton2017inductive} as an efficient GNN backbone to assess how standard GNN models handle feature heterogeneity. We also include more advanced models, namely the GCNMF model~\cite{taguchi2021graph} and the PAGNN model~\cite{jiang2020incomplete} as baselines. Both models propose new GNN architectures that incorporate feature imputation techniques to address feature heterogeneity. Finally, 
% In our experiments, we use GRAFENNE~\cite{gupta2023grafenne} as the baseline, whose limitations have been discussed previously. It is the only existing method for this problem. 

\vspace{0.5em}\noindent\textbf{Training and evaluation.} To benchmark \Ours{} and compare it with the baseline methods, we use \emph{node classification} tasks with performance measured by classification accuracy on \textsc{Actor}, \textsc{Chameleon-F}, \textsc{Squirrel-F}, \textsc{Cora}, and \textsc{CiteSeer} and by ROC-AUC on \textsc{Minesweeper} following \citet{platonov2023critical}. For all baseline methods, we use the hyperparameters provided by the authors. For the evaluation on \textsc{Actor}, \textsc{Chameleon-F}, and \textsc{Squirrel-F}, we generate 10 random splits with a ratio of $48\%/32\%/20\%$ as the training/validation/test set, respectively, following \citet{gu2024universal}. For the evaluation on \textsc{Minesweeper}, we directly utilize the 10 random splits provided by the original paper \citep{platonov2023critical}. For the evaluation on \textsc{Cora} and \textsc{CiteSeer}, we follow~\citet{luan2021heterophily} and \citet{chien2020adaptive} to randomly generate 10 random splits with a ratio of $60\%/20\%/20\%$ as the training/validation/test set, respectively. For each experiment, we report the mean and the standard deviation of the performance metric across the corresponding 10 random splits. Please see Appendix~\ref{app:settings} for additional experimental settings.

% For data splits, we use the splits given by \cite{platonov2023critical} to evaluate Minesweeper, use 10 randomly generated splits with  and randomly generate 10 random splits with ratio  We construct 10 random splits for each dataset with a split ratio of $60\%/20\%/20\%$ for train/val/test set. 

% In all experiments, we fix the number of GNN layers to be 3, the number of training epochs to be 200, and the learning rate to be $1e-3$, following the convention of~\cite{gupta2023grafenne} for a fair comparison. For the feature-feature relations, we choose to augment the original transformed graph by connecting  the top-10 similar feature nodes across all settings.

% For GNN models, we consider the GCNMF model~\cite{taguchi2021graph} and PAGNN model~\cite{jiang2020incomplete} as baselines, which both propose new GNNs applying feature imputation techniques to deal with the feature heterogeneity. In order to probe the ability of classical GNN models dealing with feature heterogeneity, we select GraphSAGE~\cite{hamilton2017inductive} as an efficient GNN backbone. We first choose the popular \textit{node classification} task to benchmark our proposed method. The performance of node classification is by accuracy. We perform a $60\%-20\%-20\%$ data split for train-validation-test set and these sets are generated at random. For all experiments, the number of GNN layers is set to 2, number of training epochs is set to 200, and the learning rate is set to 1e-3. 

\vspace{0.5em}\noindent

\subsection{Main Results}

To answer RQ\ref{rq:main}, we compare the proposed method \Ours{} with 27 state-of-the-art methods on six heterophilic and homophilic graphs. The results are shown in Table~\ref{tab:main-res}.

% \textbf{Experimental Results.}
% \paragraph{Experiment Results.}
% Table~\ref{tab:main-res} presents the results of various methods evaluated across the three datasets. GRAFENNE demonstrates strong baseline performance, consistently achieving the highest accuracy. However, our proposed method proves to be significantly more scalable and efficient, with only a minimal decrease in performance.

% As shown in Table~\ref{tab:exp:main}, \Ours{} consistently outperforms prior state-of-the-art GNN methods by a significant margin for heterophilic datasets while maintaining comparable performances on homophilic datasets. 

% \paragraph{}

% GPT: 

\begin{table*}[t]%{\linewidth}
\centering
\caption{Comparison with existing methods. \Ours{} \emph{significantly} outperforms state-of-the-art methods on heterophilic graphs while achieving \emph{comparable} accuracy with state-of-the-art methods on homophilic graphs.  Best results are marked in \textbf{bold}, and second best results are \underline{underlined}.
%\hh{underline the 2nd best method}\RZ{added}
}
\label{tab:main-res}
\resizebox{1\linewidth}!{\begin{tabular}{l|cccc|cc}
\specialrule{2.4\arrayrulewidth}{5\arrayrulewidth}{5\arrayrulewidth}
\multirow{2.2}*{\textbf{Method}}&\multicolumn{4}{c|}{\textbf{Heterophilic Graphs}}&\multicolumn{2}{c}{\textbf{Homophilic Graphs}}\\
&\textsc{Actor}&\textsc{Squirrel-F}&\textsc{Chameleon-F}&\textsc{Minesweeper}&\textsc{Cora}&\textsc{CiteSeer}\\
\midrule
{MLP} & 35.04\PM1.53 & 33.91\PM1.55 & 38.44\PM5.14 & 50.99\PM1.47 & 75.45\PM1.88 & 71.53\PM0.70\\
\midrule
{ChebNet} & 34.40\PM1.18 & 31.75\PM3.42 & 34.30\PM4.33 & \underline{91.60}\PM0.44 & 81.58\PM5.09 & 65.18\PM8.37\\
{GCN} & 30.21\PM0.86 & 35.57\PM1.86 & 40.06\PM4.38 & 72.32\PM0.93 & 87.50\PM1.68 & 75.77\PM0.96\\
{SGC} & 29.26\PM1.41 & 38.27\PM2.16 & 41.40\PM4.91 & 72.11\PM0.95 & 88.05\PM2.08 & 75.80\PM1.75\\
{GAT} & 28.86\PM0.99 & 32.74\PM3.02 & 40.11\PM2.80 & 87.59\PM1.35 & 87.11\PM1.48 & 76.43\PM1.31\\
{GraphSAGE} & 34.95\PM1.06 & 34.43\PM2.68 & 39.33\PM4.53 & 90.54\PM0.66 & 87.90\PM1.73 & 76.43\PM1.19\\
{GIN} & 28.29\PM1.45 & 39.51\PM2.83 & 40.17\PM4.76 & 75.89\PM2.09 & 85.65\PM2.26 & 72.55\PM1.78\\
{APPNP} & 33.68\PM1.26 & 33.75\PM2.31 & 37.93\PM4.33 & 67.36\PM1.08 & 87.59\PM1.68 & 75.90\PM0.91\\
{GCNII} & 34.78\PM1.50 & 35.93\PM2.87 & 41.56\PM2.74 & 88.42\PM0.85 & 87.20\PM1.56 & 73.84\PM0.91\\
{GATv2} & 28.87\PM1.39 & 32.49\PM2.51 & 39.72\PM6.60 & 88.85\PM1.16 & 87.66\PM1.52 & 76.59\PM1.19\\
{MixHop} & 35.40\PM1.34 & 30.43\PM2.33 & 37.93\PM3.87 & 89.68\PM0.57 & 84.53\PM1.53 & 76.11\PM0.83\\
{TAGCN} & 34.92\PM1.19 & 33.33\PM2.37 & 41.01\PM3.77 & 91.54\PM0.56 & 88.38\PM1.95 & 76.49\PM1.41\\
{DAGNN} & 33.15\PM1.14 & 34.72\PM2.55 & 38.94\PM3.53 & 67.87\PM1.26 & 88.27\PM1.53 & 75.81\PM0.90\\
{JKNet} & 28.63\PM0.94 & \underline{40.81}\PM2.60 & 40.39\PM4.85 & 81.00\PM0.92 & 86.24\PM0.85 & 73.11\PM1.82\\
{Virtual Node} & 30.71\PM0.82 & 38.00\PM2.28 & 41.45\PM5.46 & 72.36\PM0.98 & 87.24\PM2.00 & 69.80\PM6.89\\
\midrule
{H2GCN} & 34.20\PM1.47 & 34.02\PM3.15 & 40.89\PM3.13 & 87.08\PM0.82 & 76.89\PM2.25 & 75.87\PM1.02\\
{FSGNN}&35.60\PM1.34&37.28\PM2.63&\underline{43.30}\PM3.62&50.00\PM0.00&87.81\PM1.96&76.77\PM1.13\\
{ACM-GNN}&34.04\PM1.25&33.91\PM2.28&39.78\PM4.58&86.35\PM0.99&88.58\PM1.90&76.47\PM0.99\\
{FAGCN} & \underline{36.18}\PM1.52 & 36.52\PM1.72 & 39.83\PM3.93 & 84.69\PM2.05 & 88.66\PM2.11 & \underline{76.82}\PM1.48\\
{OrderedGNN} & 35.64\PM0.98 & 32.70\PM2.42 & 38.38\PM3.65 & 91.01\PM0.50 & 84.81\PM1.67 & 74.10\PM1.62\\
{GloGNN} & 19.80\PM2.61 & 28.72\PM2.63 & 40.17\PM4.66 & 53.42\PM1.47 & 73.02\PM2.98 & 72.46\PM2.09\\
{GGCN} & 32.76\PM1.39 & 35.06\PM5.65 & 34.08\PM3.44 & 84.76\PM1.84 & 86.39\PM1.93 & 75.36\PM1.99\\
{GPRGNN} & 35.42\PM1.33 & 34.97\PM2.83 & 40.50\PM4.55 & 83.94\PM0.98 & \textbf{88.86}\PM1.42 & 76.49\PM1.00\\
{ALT} & 33.10\PM1.38 & 37.28\PM1.49 & 39.61\PM3.36 & 89.06\PM0.64 & \underline{88.82}\PM2.02 & \textbf{76.88}\PM1.20\\
\midrule
{NodeFormer} & 29.26\PM2.31 & 24.29\PM2.60 & 34.92\PM4.08 & 77.71\PM3.50 & 87.44\PM1.37 & 75.20\PM1.27\\
{SGFormer} & 25.89\PM0.80 & 34.54\PM2.96 & 42.79\PM4.06 & 52.06\PM0.50 & 86.24\PM1.58 & 70.74\PM1.25\\
{DIFFormer} & 26.31\PM1.19 & 33.17\PM2.84 & 39.16\PM4.10 & 69.25\PM0.93 & 86.61\PM3.04 & 76.65\PM1.52\\
\midrule
\textbf{\Ours{} }(Ours)&\textbf{37.69}\PM1.57&\textbf{43.06}\PM2.89&\textbf{45.08}\PM4.04&\textbf{94.78}\PM0.41&88.23\PM1.65&76.41\PM1.57\\
\specialrule{3\arrayrulewidth}{5\arrayrulewidth}{5\arrayrulewidth}
\end{tabular}}
% \vspace{-1em}
\end{table*}

As shown in Table~\ref{tab:main-res}, our \Ours{} achieves significant performance gains (p-value$<$0.1) over prior state-of-the-art GNN methods on heterophilic graphs while maintaining competitive accuracy on homophilic graphs. Specifically, \Ours{} outperforms the best baseline methods by $4.17\%, 5.23\%, 5.35\%$ and $3.47 \%$ on \textsc{Actor}, \textsc{Squirrel-F}, \textsc{Chemeleon-F} and \textsc{Minesweeper}, respectively. While some existing models perform well on individual datasets, they often struggle on others, highlighting their insufficient consistency. In contrast, \Ours{} demonstrates the best results across all four heterophilic benchmarks. Another interesting observation is that while \Ours{} is built upon FAGCN~\citep{bo2021beyond}, it significantly surpasses FAGCN, demonstrating the beneficial effect of our graph transformation and feature edges.

\vspace{0.5em}\noindent\textbf{Discussion.} It is worth noting that most of the baseline methods cannot achieve better results compared to MLP on \textsc{Actor}, which can be explained by the fact that these methods typically treat node features and graph structure as joint input without explicitly decoupling them. The weak structural homophily exhibited by \textsc{Actor} makes typical GNNs fail to capture important feature signals, reinforcing the importance of our graph transformation strategy that boosts \textit{feature homophily} significantly. For \textsc{Squirrel-F}, we find that JKNet is the best among baselines. This observation reveals that structure information is very important within \textsc{Squirrel-F} since JKNet aggregates feature knowledge from multi-hop neighbors to learn structure-aware representation. This finding also explains the success of \Ours{} since the useful multi-hop information in \textsc{Squirrel-F} can be propagated even more efficiently through the constructed \textit{feature edges}. 

As another example, SGFormer performs the best on \textsc{Chameleon-F} among baseline methods. We argue that \textsc{Chameleon-F} needs a considerable amount of global messages and graph transformers are experts at capturing this type of information. Compared with NodeFormer and DIFFormer, SGFormer is the most advanced graph transformer utilizing simplified graph attention that strikes a good balance between global structural information and feature signal, preventing the over-globalizing issue \citep{xing2024less}. Similarly, \Ours{} transforms the original graph into a form that facilitates global message exchange by the introduction of \textit{feature edges}. As a final remark, although \Ours{} is designed specifically to deal with heterophilic datasets, \Ours{} still maintains competitive accuracy on homophilic datasets (\textsc{Cora} and \textsc{CiteSeer}), achieving results that are on par with the best existing methods. 
% This finding validate the effectiveness of \Ours{} in handling diverse graph structures, offering a strong advantage over specialized methods that perform well only in selective cases.

% \hh{this subsection is a bit short, if we still have time (and energy), let's elaborate it a bit: e.g., the improvement over the best competitors on four heterophlic datasets; we are almost as good as the best method on two homophlic datasets, why MLP performs well on actor (this can corroborate the importance of our graph transformation method right?), why graphite is built upon fagcn but is significantly better than fagcn (thanks to our graph transformation?), etc.}
\begin{wraptable}{r}{0.5\linewidth}\vspace{-2em}
% \begin{table}[t]
\caption{Relative improvement ratio of \textit{feature homophily} and \textit{adjusted homophily} across datasets. \Ours{} significantly boosts both homophily metrics. See Figure~\ref{fig:homo-analysis} for visualization.}
\label{tab:homo-ratio-increment}
\centering
\resizebox{0.9\linewidth}{!}{\begin{tabular}{lrr}%0.42
\toprule
\textbf{Dataset} & $H^{\textnormal{feature}}(\mathcal{G})$   & $ H^{\textnormal{adjusted}}(\mathcal{G})$      \\
\midrule
\textsc{Actor}       &   +179\%      &   +2767\% \\
\textsc{Squirrel-F}  &   +961\%      &    +215\% \\
\textsc{Chameleon-F} &  +1739\%      &    +402\% \\
\textsc{Minesweeper} &    +41\%      &   +1023\% \\
\bottomrule
\end{tabular}}
\vspace{-1em}
% \end{table}
\end{wraptable}
\subsection{Homophily Analysis}\label{ssec:exp:homo}
% In this subsection, we first formally introduce two homophily metrics we use in Figure~\ref{fig:homo-analysis} and then conduct a thorough analysis on the different increment ratios across these heterophilic graphs.
To answer RQ\ref{rq:homo}, we conduct a homophily analysis across heterophilic datasets under two homophily metrics: \textit{feature homophily} $H^\textnormal{feature}(\CAL G)$ and \textit{adjusted homophily} $H^\textnormal{adjusted}(\CAL G)$ (see Appendix~\ref{app:homo-metric} for their formal definitions). Table~\ref{tab:homo-ratio-increment} and Figure~\ref{fig:homo-analysis} show the relative improvements between the homophily metrics before and after applying \Ours{}. We can observe a significant boost in both homophilily metrics after applying \Ours{} across the four heterophilic datasets.

\vspace{0.5em}\noindent\textbf{Discussion.} Overall, \Ours{} 
effectively boosts both homophily metrics across all heterophilic datasets. Specifically, Squirrel-F and Chameleon-F demonstrate significant boosts in terms of \textit{feature homophily}. This is mainly because their discrete features directly correspond to specific topics and each feature edge will contribute much higher feature similarity than usual edges. On the other hand, Actor and Minesweeper showcase much higher \textit{adjusted homophily} after applying \Ours{}. For Actor, this favorable behavior can be attributed to the high correlation between page co-occurrences and node labels; while for Minesweeper, the sum of label-specific node degrees (defined in Equation~(\ref{eq:adj-homo})) increases much due to the transformation performed by \Ours{}.

% In contrast, Actor does not exhibit high improvement in \textit{feature homophily}, partly because co-occurance on Wikipedia pages reveals an already mixed feature distribution and \textit{feature edges} do not necessarily help create stronger clusters of similar graph nodes. On the other hand, the same trend observed on Minesweeper dataset can be attributed to the fact that $50\%$ of the nodes have random features upon construction~\cite{platonov2024characterizing}, which impedes beneficial homophily boosting in terms of feature similarity. 

\vspace{0.5em}\noindent\textbf{Baseline methods.} 
In our experiments, we consider a wide range of GNN baselines, including MLP (structure-agnostic), homophilic GNNs, heterophilic GNNs, and Graph Transformers. The full list is shown in
Appendix~\ref{app:baseline}.
% Table~\ref{tab:baselines}. 
% Please see the appendix for descriptions of baseline methods.
%Note that due to computational limitations, we do not perform an exhaustive hyperparameter search for all baseline methods. We instead rely on default settings or previously reported optimal configurations for each dataset, which have been shown to perform well for the corresponding dataset. 

% \begin{wraptable}{r}{0.6\linewidth}
\begin{table}[t]
\caption{Effectiveness of the proposed graph transformation. \Ours{} transformed graphs alone can already enhance the performance of homophilic GNNs.
%\hh{this result is great. i totally did not expect you could do this experiments in such a short notice. thank you for your dedication and hard work}\RZ{Thank you too for your great suggestion!}
}
\label{tab:abla}
\centering
\resizebox{0.7\linewidth}{!}{\begin{tabular}{l|cc|cc}
\toprule
\textbf{Dataset} & \multicolumn{2}{c|}{\textsc{Actor}}  & \multicolumn{2}{c}{\textsc{Minesweeper}}  \\
\small{+\Ours{}}?&\xmark&\cmark&\xmark&\cmark\\
\midrule
GCN      &30.21\PM0.86&\textbf{34.83}\PM1.28&72.32\PM0.93&\textbf{75.38}\PM1.56\\
GAT      &28.86\PM0.99&\textbf{32.09}\PM1.35&87.59\PM1.35&\textbf{88.66}\PM0.88\\
GraphSAGE&34.95\PM1.06&\textbf{35.09}\PM1.06&90.54\PM0.66&\textbf{90.85}\PM0.67\\
JKNet    &28.63\PM0.94&\textbf{35.96}\PM1.40&81.00\PM0.92&\textbf{85.56}\PM2.59\\
GIN      &28.29\PM1.45&\textbf{33.75}\PM1.83&75.89\PM2.09&\textbf{87.07}\PM1.71\\
\bottomrule
\end{tabular}}
\vspace{-0.5em}
\end{table}
% \end{wraptable}

\begin{table}[t]
\centering
\caption{Comparison of aggregators for the features of feature nodes. All aggregators in fact perform similarly, so we choose averaging due to its simplicity and efficiency.}
\label{tab:exp-aggr}
\resizebox{0.9\linewidth}{!}{\begin{tabular}{l|cccc}
\toprule
\textbf{Dataset} & Averaging (Ours) & Learned Embeddings & Learned Attention & Majority Voting \\
\midrule
\textsc{Actor}        & 37.69 & 37.46 & 37.13 & 37.59 \\
\textsc{Squirrel-F}   & 43.06 & 43.53 & 43.46 & 43.28 \\
\textsc{Chameleon-F}  & 45.08 & 44.80 & 44.36 & 45.64 \\
\textsc{Minesweeper}  & 94.78 & 94.47 & 94.56 & 94.75 \\
\bottomrule
\end{tabular}}
\vspace{-1em}
\end{table}

\begin{table}[t]
\centering
\caption{Running time comparison with and without \Ours{}. Graphs are sorted in decreasing order by graph sizes. Our graph transformation has only a small impact on running time while significantly improves accuracy.}
\label{tab:exp-time}
\resizebox{0.75\linewidth}{!}{\begin{tabular}{lcccc}
\toprule
\textbf{Method} & \textsc{Minesweeper} & \textsc{Actor} & \textsc{Squirrel-F} & \textsc{Chameleon-F} \\
\midrule
No transformation    & 1.9\,min & 1.5\,min & 0.7\,min & 0.5\,min \\
With transformation  & 2.3\,min & 2.0\,min & 1.1\,min & 0.7\,min \\
\bottomrule
\end{tabular}}
\end{table}

\begin{table}[t]
\centering
\caption{Sensitivity analysis of hyperparameters $\tau$ and $w_{\mathcal X}$ on \textsc{Minesweeper}. Our method consistently outperform the best baseline under various hyperparameters.}
\label{tab:exp-hyperparams}
\resizebox{0.65\linewidth}{!}{\begin{tabular}{c|ccccc}
\toprule
\multirow{2}{*}{\makecell[l]{Best\\Baseline}} & \multicolumn{5}{c}{\Ours{} (Ours)} \\
& $\tau=0.1$ & $\tau=0.5$ & $\tau=1.0$ & $\tau=1.5$ & $\tau=2.0$ \\
\midrule
91.60 & 93.48 & 94.29 & 94.78 & 94.66 & 94.18 \\
\midrule
\multirow{2}{*}{\makecell[l]{Best\\Baseline}} & \multicolumn{5}{c}{\Ours{} (Ours)} \\
& $w_{\mathcal X}=0.1$ & $w_{\mathcal X}=0.25$ & $w_{\mathcal X}=0.5$ & $w_{\mathcal X}=0.75$ & $w_{\mathcal X}=1.0$ \\
\midrule
91.60 & 94.78 & 94.16 & 93.95 & 93.53 & 93.15 \\
\bottomrule
\end{tabular}}
\end{table}

\subsection{Ablation Studies}\label{ssec:exp:abla}
\textbf{Graph transformation.} To further demonstrate the effectiveness of our proposed graph transformation \Ours{} and answer RQ\ref{rq:abla}, we compare the performance of homophilic GNNs on the original graph and that on the transformed graph. In this experiment, we use two larger-scale datasets, \textsc{Actor} and \textsc{Minesweeper}, and five representative homophilic GNNs, GCN, GAT, GraphSAGE, JKNet, and GIN. The results are presented in Table~\ref{tab:abla}. 

From Table~\ref{tab:abla}, we can see that our proposed \Ours{} consistently improves the performance of the five representative homophilic GNNs on both datasets, even though these GNNs are not specially designed for modeling feature nodes. For example, the accuracy of GAT on \textsc{Actor} is enhanced from 30.21\% to 34.83\%, which is a relative improvement of 15.29\%. The results demonstrate that our proposed graph transformation \Ours{} can significantly enhance the performance of homophilic GNNs on originally heterophilic graphs, echoing the fact that our proposed graph transformation can significantly increase the graph homophily. 

\textbf{Features of feature nodes.} \Ours{} uses averaging aggregation to define the features of feature nodes. To elaborate on the rationale of this simple aggregator and answer RQ\ref{rq:feat}, we compare with other aggregators on heterophilic datasets. The results are shown in Table~\ref{tab:exp-aggr}. In fact, all aggregators have similar accuracies while averaging is more efficient than learned embeddings and attention-weighted aggregation and is simpler than majority voting. Therefore, we use averaging in our method due to its simplicity and efficiency.

\textbf{Computational efficiency.} To evaluate the efficiency of our method and answer RQ\ref{rq:eff}, we provide a comparison of running times with and without our transformation, respectively. The results are shown in Table~\ref{tab:exp-time} (sorted by graph sizes). We can see that our graph transformation has only a small impact on running time while significantly improving accuracy. Notably, the computational overhead gets smaller on larger graphs (e.g., Minesweeper), justifying the scalability of our method. This is because the number of added nodes (i.e., the number of features) is typically negligibly small compared with the number of nodes on large graphs, and the number of feature edges is proportional to the space complexity of the node feature matrix $\BM X$.

\textbf{Hyperparameters.} Since our method has a few hyperparameters including $\tau$ and $w_\CAL X$, to answer RQ\ref{rq:hyperparams}, we provide a sensitivity analysis of them on \textsc{Minesweeper}. The results are shown in Table~\ref{tab:exp-hyperparams}. From the table, we can see that our method is not sensitive to these hyperparameters, and our method consistently outperform the best baseline under various hyperparameter values.

% \textbf{Memory:} 
% Table~\ref{tab:memory_table} highlights the memory consumption of our method compared to existing approaches. Our model demonstrates substantial reductions in memory usage (for larger-scale datasets with rich feature sets including Cora, Citeseer, Actor and Squirrel), enabling scalability to larger datasets and more complex graph structures while our method maintains comparable memory consumption for other graph datsets with smaller size. This efficiency stems from our optimized architecture and reduced reliance on intermediate storage during computation, making it particularly suitable for resource-constrained environments. 

% \textbf{Convergence Speed and Performance:}
% Figure~\ref{fig:main} and Table~\ref{quantitative_table} collectively showcase the effectiveness of our method in terms of convergence speed and accuracy. Our approach not only converges significantly faster than competing methods, as evidenced by the steep decline in the training loss curve during initial epochs, but it also consistently outperforms baseline models across all datasets. Table~\ref{quantitative_table} provides a detailed comparison, showing that our method achieves higher accuracy even when competing methods, such as GRAFENNE, are allocated significantly more computational time (e.g., 2x or 3x). These results highlight the robustness and efficiency of our approach in both homophilic and heterophilic graph settings, attributed to its novel feature aggregation strategy and effective representation learning.

\section{Related Work}
% \hh{weird place to put related work section}\RZ{agreed. I have moved the related work section}

% Prominent spectral-based solutions against over-squashing include \cite{arnaiz2022diffwire,deac2022expander,karhadkar2022fosr}. More recently, \cite{di2023over} thoroughly analyzes factors that contribute to over-squashing. 
% Additional rewriting methods and more advanced message-passing paradigms include~\cite{barbero2023locality,qian2023probabilistically,behrouz2024graph}.

% These approaches collectively advance our understanding of oversquashing and provide structural interventions to improve message-passing efficiency.

% \cite{alon2020bottleneck,shi2023exposition} shows the ``over-squashing'' problem in Message Passing Neural Networks (MPNN), where long-range information dissemination is hindered.
% Prior works have attempted to identify topological bottlenecks and directly modify graph connectivity with various optimization techniques. \cite{topping2021understanding} lays the initial framework in connecting over-squashing with graph Ricci curvature. Follow-up works like \cite{nguyen2023revisiting, shi2023curvature} proposes rewring methods based on the same philosophy of connecting with graph curvatures.   \cite{black2023understanding} further analyze over-squashing with the tool of effective resistance. 

A substantial body of research has explored the challenges of heterophily in graph neural networks (GNNs). Many early approaches sought to improve information aggregation, such as MixHop~\citep{abu2019mixhop}, which mixes different-hop neighborhood features, and GPRGNN~\citep{chien2020adaptive}, which employs generalized PageRank propagation for adaptive message passing. Other methods focus on explicit heterophilic adaptations, such as H2GCN~\citep{zhu2020beyond}, which separates ego- and neighbor-embeddings, and FAGCN~\citep{bo2021beyond}, which learns optimal representations via frequency-adaptive filtering. Additional works, including OrderedGNN~\citep{song2023ordered}, GloGNN~\citep{li2022finding}, and GGCN~\citep{yan2022two}, leverage structural ordering, global context, and edge corrections, respectively, to enhance performance on heterophilic graphs. Recent advances explore alternative formulations, such as component-wise signal decomposition (e.g. ALT, \citealp{xu2023node}) and adaptive residual mechanisms~\citep{xu2024slog, yan2024trainable} for greater flexibility. Beyond architectural innovations, rigorous benchmarking efforts~\citep{lim2021large, zhu2024impact, platonov2023critical} have been introduced to standardize evaluations and assess generalization across diverse graph properties. A broader synthesis of heterophilic GNN techniques can be found in recent surveys~\citep{zheng2022graph, zhu2023heterophily, luan2024heterophilic, gong2024survey}. Please refer to Appendix~\ref{app:related} for additional related work.

\section{Conclusion \& Future Work}
% Our proposed approach addresses critical challenges in graph neural network (GNN) performance by introducing a new method for handling incomplete node feature sets. Through comprehensive experimental evaluation, we demonstrate significant improvements across diverse graph datasets, achieving less memory consumption, faster convergence, and superior accuracy compared to the state-of-the-art method. Our method effectively mitigates the limitations of traditional GNNs in face of heterogeneous feature sets, offering a promising solution for real-world graph learning challenges. These contributions provide practical insights for applications in many real-world domains naturally manifesting the feature heterogeneity problem.

In this paper, we propose \Ours{}, a simple yet efficient framework to address the heterophily issue in node classification. By introducing feature nodes that connect to graph nodes with corresponding discrete features, we can solve the heterophily issue by increasing the graph homophily ratio. 
Through theoretical analysis and empirical study, we validate that \Ours{} can indeed effectively increase the graph homophily. Our extensive experiments demonstrate that \Ours{} consistently outperforms state-of-the-art methods on heterophilic graph datasets and achieves comparable performance on homophilic graph datasets. An interesting future direction would be extending the proposed graph transformation to general graphs with continuous node features; possible approaches include clustering the continuous features into a few clusters and binning the continuous features into discrete buckets. Other future directions include (i) studying how our graph transformation affects graph properties, (ii) connecting to the node distinguishability principle \citep{luan2023graph}, and (iii) identifying an optimal subset of features \citep{zheng2025let}.

\subsubsection*{Acknowledgements} \thanks{This work is supported by NSF (2416070).
The content of the information in this document does not necessarily reflect the position or the policy of the Government, and no official endorsement should be inferred.  The U.S. Government is authorized to reproduce and distribute reprints for Government purposes notwithstanding any copyright notation hereon.
}

\newpage
\subsubsection*{Ethics Statement}
Our study is based entirely on publicly available graph datasets commonly used in the GNN literature and does not involve private or sensitive information. We develop a graph transformation framework that explicitly increases graph homophily to enable more effective message passing. To ensure methodological soundness and reproducibility, we provide both theoretical analyses and extensive empirical evaluations across heterophilic datasets. The release of code and data splits is intended solely for academic research to advance the understanding of graph machine learning and to support future work on graph neural networks, and are not designed for sensitive or high-stakes applications.

\subsubsection*{Reproducibility Statement}
We include the conceptual framework, transformation steps, method details and evaluation setup in the paper and appendix. To facilitate reproducibility, we also publicly release our code at \url{https://github.com/q-rz/ICLR26-GRAPHITE}.

% \newpage
\bibliography{sections/80-Reference}
\bibliographystyle{iclr2026_conference}

\newpage
\appendix

\section{Experimental Settings (Cont'd)}\label{app:settings}

\subsection{Datasets (Cont'd)}\label{app:data}
For heterophilic group, we consider the following datasets, which are widely used as benchmarks for studying graph learning methods under heterophilic settings.
\begin{itemize}
\setlength\itemsep{0.5em}
    \item \textsc{Actor}~\citep{pei2020geom}: \textsc{Actor} dataset is an actor-only induced subgraph of the film dataset introduced by~\citep{tang2009social}. The nodes are actors and the edges denote co-occurrence on the same Wikipedia page. The node features are keywords on the pages and we classify nodes into five categories. 
    \item Squirrel-Filtered (\textsc{Squirrel-F}, \citealp{platonov2023critical}): \textsc{Squirrel-F} is a page-page dataset. It is a subset of the Wiki dataset~\citep{rozemberczki2021multi} that focus on the topic related to squirrel. Nodes are web pages and edges are mutual links between pages. The node features are important keywords in the pages and we classify nodes into five categories in terms of traffic of the webpage. 
    \item Chameleon-Filtered (\textsc{Chameleon-F}, \citealp{platonov2023critical}): \textsc{Chameleon-F} is a page-page dataset. It is a subset of the Wiki dataset~\citep{rozemberczki2021multi} that focus on the topic related to chameleon.  Nodes are web pages and edges are mutual links between pages. The node features are important keywords in the pages and we classify nodes into five categories in terms of traffic of the webpage. 
    \item \textsc{Minesweeper}~\citep{platonov2023critical}: \textsc{Minesweeper} is a synthetic dataset that simulates a Minesweeper game with 100$\times$100 grid. Each node is connected to its neighboring nodes where $20\%$ nodes are selected as mines at random. Node features are numbers of neighboring mines and the goal is to predict whether each test node is mine. These datasets are widely used as benchmarks for studying graph learning methods under heterophilic settings.
\end{itemize}
    For the homophilic group, we consider the following datasets, which are standard homophilic network benchmarks.
\begin{itemize}

    \item \textsc{Cora}~\citep{sen2008collective} : Cora dataset is a citation network, where nodes represent scientific papers in the machine learning field, and edges correspond to citation relationships between these papers. Each node is associated with a set of features that describe the paper, represented as a bag-of-words model. The task for this dataset is to classify each paper into one of seven categories, reflecting the area of research the paper belongs to.
    
    \item \textsc{CiteSeer}~\citep{sen2008collective}:
    CiteSeer dataset is a citation network of scientific papers. It consists of research papers as nodes, with citation links forming the edges between them. Each node is associated with a set of features derived from the paper's content, which is a bag-of-words representation of the paper's text. The task for this dataset is to classify each paper into one of six categories, each representing a specific field of study. 
\end{itemize}

\subsection{Baseline Methods (Cont'd)}
\label{app:baseline}
We briefly introduce GNN-based baseline methods as follows. % (excluding structure-agnostic methods). 

% \subsection*{Spectral and Spatial GNNs}
The first category is \textit{homophilic GNNs}, which are originally designed under the homophily assumption.
\begin{itemize}
    \item ChebNet~\citep{defferrard2016convolutional}: Uses Chebyshev polynomials to approximate graph convolutions.
    \item GCN~\citep{kipf2016semi}: Employs a first-order Chebyshev approximation for spectral graph convolutions.
    \item SGC~\citep{wu2019simplifying}: Simplifies GCN by removing non-linearities and collapsing weight matrices for efficiency.
    \item GAT~\citep{velivckovic2018graph}: Introduces attention mechanisms to assign adaptive importance to edges.
    \item GraphSAGE~\citep{hamilton2017inductive}: Uses several aggregators for inductive graph learning.
    \item GIN~\citep{xu2018powerful}: Employs sum-based aggregation to maximize graph structure expressiveness.
    \item APPNP~\citep{gasteiger2018predict}: Combines personalized PageRank with neural propagation.
    \item GCNII~\citep{chen2020simple}: Extends GCN with residual connections and identity mapping for deep GNN training.
    \item GATv2~\citep{brody2021attentive}: Enhances GAT with dynamic attention coefficients for flexible neighbor weighting.
    \item MixHop~\citep{abu2019mixhop}: Aggregates multi-hop neighborhood features by mixing different powers of adjacency matrices.
    \item TAGCN~\citep{du2017topology}: Introduces trainable polynomial filters for adaptive, multi-scale feature extraction.
    \item DAGNN~\citep{liu2020towards}: Uses dual attention to decouple message aggregation and transformation, improving depth scalability.
    \item JKNet~\citep{xu2018representation}: Uses a jumping knowledge mechanism to combine features from different layers adaptively. We default the backbone GNN model to GCN.
    \item Virtual Node~\citep{gilmer2017neural}: Introduces an auxiliary global node to facilitate message passing. We default the backbone GNN model to GCN.
\end{itemize}

The second category is \textit{heterophilic GNN}s, which are designed for graphs where connected nodes often have different labels.

% \subsection*{Heterophilic GNNs}
\begin{itemize}
    \item H2GCN~\citep{zhu2020beyond}: Enhances GNNs by ego-/neighbor-embedding seperation, higher-order neighbors and intermediate representation combinations.
    \item FSGNN~\citep{maurya2021improving}: Employs soft feature selection and hop normalization over GNN layers to form a simple, shallow GNN. We use their default 3-hop variant.
    \item ACM-GNN~\citep{luan2022revisiting}: Introduces adaptive channel mixing to diversify local information. We use their default ACM-GCN+ variant.
    \item FAGCN~\citep{bo2021beyond}: Uses frequency adaptive filtering to learn optimal graph representations.
    \item OrderedGNN~\citep{song2023ordered}: Aligns the order to encode neighborhood information and avoids feature mixing.
    \item GloGNN~\citep{li2022finding}: Incorporates global structural information to enhance graph learning beyond local neighborhoods.
    \item GGCN~\citep{yan2022two}: Utilizes structure/feature-based edge correction to combat over-smoothing and heterophily.
    \item GPRGNN~\citep{chien2020adaptive}: Introduces generalized PageRank propagation to  capture the graph structure.
    \item ALT~\citep{xu2023node}: Decomposes graph into components, extracts signals from these components, and adaptively integrate these signals.
\end{itemize}

% \subsection*{Graph Transformers}
The last category is \textit{graph transformers}, which adapt transformer architectures to graph data and look beyond local neighborhood aggregation. 
\begin{itemize}
    \item NodeFormer~\citep{wu2022nodeformer}:  Introduces all-pair message passing on layer-specific adaptive latent graphs, enabling global feature propagation with linear complexity.
    \item SGFormer~\citep{wu2024simplifying}:  Develops a graph encoder backbone that efficiently computes all-pair interactions with one-layer attentive propagation.
    \item DIFFormer~\citep{wu2023difformer}:
    Proposes an energy-constrained diffusion model, leading to variants that are efficient and capable of capturing complex structures.
\end{itemize}

% \tw{add a table to include all baselines like IGNN}

% Other baselines:
% https://ieeexplore.ieee.org/stamp/stamp.jsp?arnumber=10498196

\subsection{Training \& Evaluation (Cont'd)}
For our method, we use $w_\CAL X\in\{0.01,0.1,0.6,8\}$, $w_0\in\{0.1,0.2,0.3,0.5,1,8\}$, $\tau\in\{0.01,0.1,1\}$, and dropout rate 0.2. We use the GNN architecture described in the method section %Section~\ref{ssec:method-detail}
with 8 GNN layers with hidden dimensionality 512 and add a two-layer MLP after each GNN layer for heterophilic graphs and use FAGCN for homophilic graphs. We use original node features as described in Section~\ref{ssec:method}, except that we use zeros as the features of graph nodes on Squirrel-F and that we normalize the features of graph nodes on Cora and CiteSeer after computing the features of feature nodes. We train the GNN with learning rate 0.00003 for 1000 steps using the Adam optimizer \citep{kingma2014adam}. Our method was implemented in PyTorch 2.7.0 and Deep Graph Library (DGL) 2.4.0, and experiments were run on Intel Xeon CPU @ 2.20GHz with 96GB memory and NVIDIA Tesla V100 32GB GPU.

\section{Definition of Homophily Metrics}
\label{app:homo-metric}

To measure to what extent \Ours{} can  boost graph homophily on heterophlic datasets, we consider two popular homophily metrics: \textit{feature homophily}~\citep{jin2022raw} and \textit{adjusted homophily}~\citep{platonov2024characterizing}. Formally, given a graph $\mathcal{G}$, \textit{feature homophily} $H^{\textnormal{feature}}$ is defined as follows:
\begin{equation}
\label{eq:homo-equation}
    H^{\textnormal{feature}} (\mathcal{G}) := \frac1{|\mathcal{E}|}\sum_{(v_i,v_j)\in \mathcal{E}} \textnormal{sim}(v_i,v_j),
\end{equation}
where $\textnormal{sim}(v_i,v_j) := \OP{cos}(\BM X[v_i,:], \BM X[v_j,:])$ is the cosine-similarity computed between features of nodes $v_i,v_j$. This metric is a variant of the \textit{generalized edge homophily ratio} $H^{\textnormal{edge}}$ proposed by~\citep{jin2022raw}, which measures the feature similarity between each of the connected node pairs in the graph dataset. Then, the \textit{adjusted homophily} $H^{\textnormal{adjusted}}$ is defined as follows:
\begin{equation}
    H^{\textnormal{adjusted}}(\mathcal{G}) := \frac{H^{\textnormal{edge}}(\mathcal{G}) - \sum_{c=1}^C D_c^2/(2|\mathcal{E}|)^2}{1 - \sum_{c=1}^C D_c^2/(2|\mathcal{E}|)^2},
\label{eq:adj-homo}
\end{equation}
where $C$ denotes the number of classes and $H^{\textnormal{edge}}(\mathcal{G})$ is \textit{edge homophily}, which is defined similarly as Equation~(\ref{eq:homo-equation}) with the similarity function $\textnormal{sim}(v_i,v_j) = \mathbf{1}_{\{y_{v_i}=y_{v_j}\}}$, and
\AL{D_c := \sum_{v\in\CAL V} \OP{deg}(v)1_{[y_v=c]}\label{eq:node-degree-sum}}
is the sum of degrees $\OP{deg}(v)$ of nodes with label $c$, where $y_{v}$ denotes the label of node $v$. Since we do not have node labels for the \textit{feature nodes} when computing \textit{adjusted homophily}, we assign them a ``soft label,'' which is a uniform probability distribution over the labels of its 1-hop neighbors.

% \begin{table*}[t]
% \caption{Summary of baseline methods.}
% \label{tab:baselines}
% \centering
% \resizebox{1\linewidth}{!}{
% \begin{tabular}{c|c}
% \toprule
% \textbf{Type} & \textbf{Baseline Methods}  \\
% \midrule
% \makecell{Non-Graph} & \makecell[l]{Multi-Layer Perceptron (MLP)}  \\
% \midrule
% \makecell{Homophilic\\GNNs}    & \makecell[l]{ChebNet~\cite{defferrard2016convolutional}, GCN~\cite{kipf2016semi}, SGC~\cite{wu2019simplifying},\\GAT~\cite{velivckovic2018graph},  GraphSAGE~\cite{hamilton2017inductive}, GIN~\cite{xu2018powerful},\\APPNP~\cite{gasteiger2018predict}, GCNII~\cite{chen2020simple}, GATv2~\cite{brody2021attentive},\\MixHop~\cite{abu2019mixhop}, TAGCN~\cite{du2017topology}, DAGNN~\cite{liu2020towards},\\JKNet~\cite{xu2018representation}, Virtual Node~\cite{gilmer2017neural}} \\
% \midrule

% \makecell{Heterophilic\\GNNs}  & \makecell[l]{ H2GCN~\cite{zhu2020beyond}, FAGCN~\cite{bo2021beyond}, OrderedGNN~\cite{song2023ordered}, GloGNN~\cite{li2022finding}, \\GGCN~\cite{yan2022two}, GPRGNN~\cite{chien2020adaptive}, ALT~\cite{xu2023node}}  \\
% \midrule
% \makecell{Graph\\Transformers}  & \makecell[l]{NodeFormer~\cite{wu2022nodeformer}, SGFormer~\cite{wu2024simplifying}, DIFFormer~\cite{wu2023difformer}} \\
% \bottomrule
% \end{tabular}
% }
% % \vspace{-1em}
% \end{table*}

\section{Additional Related Work}\label{app:related}

In modern machine learning research~\citep{ma2022graph,yu2026planetalign,yu2025joint,zhang2026guiding,bao2025latte,chen2024wapiti,wei2026agentic,wei2026diffkgw,wei2025cofirec,wei2025evo,wei2024robust,wei2022augmentations,cui2026adafuse,chen2026influence,liu2025seeing,liu2025breaking,liu2024logic,liu2024class,liu2024aim,liu2023topological,bartan2025fineamp,zeng2026subspace,zeng2026pave,zeng2026harnessing,zeng2025hierarchical,zeng2025interformer,zeng2024hierarchical,zeng2024graph,zeng2023parrot,zeng2023generative,zou2025transformer,zou2025latent,lin2026mixture,lin2025moralise,lin2024backtime,zhou2025geometric,jing2025trqa,qiu2026remix,qiu2025efficient,qiu2025ask,qiu2025efficient2,qiu2024tucket,qiu2023reconstructing,qiu2022dimes,xu2024discrete,li2025beyond,li2025haystack,li2025graph,li2025flow,li2023metadata,yoo2025embracing,yoo2025generalizable,yoo2024ensuring,chan2024group,wu2024fair,he2026powergrow,he2024sensitivity,wang2023networked}, 
a problem related to heterophily is over-squashing. The over-squashing problem in message passing neural networks arises when long-range information is exponentially compressed, preventing effective dissemination across the graph~\citep{alon2020bottleneck, shi2023exposition}. A primary research direction addresses this issue by identifying topological bottlenecks and modifying graph connectivity. \citet{topping2021understanding} established an initial framework linking oversquashing to graph Ricci curvature, demonstrating that negatively curved edges act as bottlenecks. Building on this idea, subsequent works have developed rewiring strategies inspired by curvature-based principles~\citep{nguyen2023revisiting, shi2023curvature}. Beyond curvature, \cite{black2023understanding} introduced a perspective using effective resistance.
Another line of research leverages spectral methods to counteract over-squashing, with notable approaches including spectral gaps~\citep{arnaiz2022diffwire}, expander graph constructions~\citep{deac2022expander}, and first-order spectral rewiring~\citep{karhadkar2022fosr}. More recently, \cite{di2023over} provided a comprehensive analysis of the factors contributing to oversquashing. Additional solutions explore advanced rewiring strategies and novel message-passing paradigms~\citep{barbero2023locality,qian2023probabilistically,behrouz2024graph}.

\section{Theoretical Analysis}\label{app:proofs}

\subsection{Assumptions}\label{app:ass}
In this subsection, we introduce the assumptions of our theoretical analysis, which are mild and realistic. 

Given a graph $\CAL G=(\CAL V,\CAL E,\BM X)$ with $\CAL E\ne\varnothing$ and $\BM X\in\{0,1\}^{\CAL V\times\CAL X}$, we define the feature similarity metric as $\OP{sim}(v_i,v_j):=\|\BM X[v_i,:]\land\BM X[v_j,:]\|_\infty$ and use the feature homophily as the homophily metric:
\AL{\OP{hom}(\CAL G):=\frac1{|\CAL E|}\sum_{(v_i,v_j)\in\CAL E}\OP{sim}(v_i,v_j).}
Furthermore, we assume that the original graph $\CAL G$ is heterophilic. That is, we have $\OP{hom}(\CAL G)<1$ while there exists a pair of nodes, $v_i,v_j\in\CAL V$ ($v_i\ne v_j$), such that $\OP{sim}(v_i,v_j)>0$ but $(v_i,v_j)\notin\CAL E$. 

Besides that, we assume that the given graph $\CAL G$ does not have too dense features. Formally, we assume that $|\CAL X|\le O(|\CAL V|)$ and that $\|\BM X\|_0\le O(|\CAL E|)$.
For the transformed graph $\CAL G^*$, we assume that every feature is used: for any feature $k\in\CAL X$, there exists a graph node $v_i\in\CAL V$ such that $\BM X[v_i,k]=1$. 

\subsection{Technical Lemma}
Here, we prove a technical lemma that we will use later.

\begin{LEM}\label{lem:add}
Let $\CAL A,\CAL B\subset\BB R$ be two nonempty, finite multisets such that
\AM{\frac1{|\CAL A|}\sum_{z\in\CAL A}z<\frac1{|\CAL B|}\sum_{z\in\CAL B}z.} Then,
\AM{\frac{1}{|\CAL A\sqcup\CAL B|}\sum_{z\in\CAL A\sqcup\CAL B}z>\frac1{|\CAL A|}\sum_{z\in\CAL A}z.}
\end{LEM}

\begin{proof}
To simplify notation, let
\AL{
\mu_A&:=\frac1{|\CAL A|}\sum_{z\in\CAL A}z,\\
\mu_B&:=\frac1{|\CAL B|}\sum_{z\in\CAL B}z,\\
\varDelta&:=\mu_\CAL B-\mu_\CAL A>0.
}
Then,
\AL{
&\frac{1}{|\CAL A\sqcup\CAL B|}\sum_{z\in\CAL A\sqcup\CAL B}z-\frac1{|\CAL A|}\sum_{z\in\CAL A}z
\\={}&\frac{1}{|\CAL A|+|\CAL B|}\bigg(\sum_{z\in\CAL A}z+\sum_{z\in\CAL B}z\bigg)-\frac1{|\CAL A|}\sum_{z\in\CAL A}z
\\={}&\frac{1}{|\CAL A|+|\CAL B|}\bigg(|\CAL A|\cdot\frac1{|\CAL A|}\sum_{z\in\CAL A}z+\sum_{z\in\CAL B}z\bigg)-\frac1{|\CAL A|}\sum_{z\in\CAL A}z
\\={}&\frac{1}{|\CAL A|+|\CAL B|}\bigg(|\CAL A|\cdot\mu_\CAL A+\sum_{z\in\CAL B}z\bigg)-\mu_\CAL A
\\={}&\frac{1}{|\CAL A|+|\CAL B|}\big(|\CAL A|\cdot\mu_\CAL A+|\CAL B|\cdot\mu_\CAL B\big)-\mu_\CAL A
\\={}&\frac{1}{|\CAL A|+|\CAL B|}\big(|\CAL A|\cdot\mu_\CAL A+|\CAL B|\cdot\mu_\CAL B-(|\CAL A|+|\CAL B|)\cdot\mu_\CAL A\big)
\\={}&\frac1{|\CAL A|+|\CAL B|}\big(|\CAL B|\cdot\mu_\CAL B-|\CAL B|\cdot\mu_\CAL A\big)
\\={}&\frac{|\CAL B|}{|\CAL A|+|\CAL B|}(\mu_\CAL B-\mu_\CAL A)
\\={}&\frac{|\CAL B|}{|\CAL A|+|\CAL B|}\varDelta>0
.}
It follows that
\AL{\frac{1}{|\CAL A\sqcup\CAL B|}\sum_{z\in\CAL A\sqcup\CAL B}z&>\frac1{|\CAL A|}\sum_{z\in\CAL A}z.\qedhere}
\end{proof}

\subsection{Proof of Theorem~\ref{thm:naive}}
\vspace{0.5em}\noindent\textbf{Homophily.} Since the original graph $\CAL G$ is homophilic, then there exists a pair of nodes, $v_i,v_j\in\CAL V$ ($v_i\ne v_j$), such that $\OP{sim}(v_i,v_j)=\|\BM X[v_i,:]\land\BM X[v_j,:]\|_\infty>0$ but $(v_i,v_j)\notin\CAL E$. According to the definition of $\CAL E^\dagger$, we know that $(v_i,v_j)\in\CAL E^\dagger\setminus\CAL E\ne\varnothing$, so $\CAL E^\dagger\setminus\CAL E\ne\varnothing$. 

Furthermore, for any $(v_i,v_j)\in\CAL E^\dagger\setminus\CAL E$, since $\OP{sim}(v_i,v_j)=\|\BM X[v_i,:]\land\BM X[v_j,:]\|_\infty>0$, then there exists a feature $k\in\CAL X$ such that $\BM X[v_i,k]\land\BM X[v_j,k]>0$. Since the feature matrix $\BM X$ is binary, then we must have
\AL{\BM X[v_i,k]=1,\qquad\BM X[v_j,k]=1.}
It follows that
\AL{\OP{sim}(v_i,v_j)&=\|\BM X[v_i,:]\land\BM X[v_j,:]\|_\infty\\
&=\max_{k'\in\CAL X}|\BM X[v_i,k']\land\BM X[v_j,k']|\\
&\ge|\BM X[v_i,k]\land\BM X[v_j,k]|\\
&=|1\land1|=1
.}
Since $\OP{hom}(\CAL G)<1$, then
\AL{\OP{sim}(v_i,v_j)\ge1>\OP{hom}(\CAL G).}
Therefore, by Lemma~\ref{lem:add} with
\AL{
\CAL A&:=\{\!\!\{\OP{sim}(v_i,v_j):(v_i,v_j)\in\CAL E\}\!\!\},\\
\CAL B&:=\{\!\!\{\OP{sim}(v_i,v_j):(v_i,v_j)\in\CAL E^\dagger\setminus\CAL E\}\!\!\},
}
we have
\AL{
\OP{hom}(\CAL G^\dagger)&=\frac1{|\CAL E^\dagger|}\sum_{(v_i,v_j)\in\CAL E^\dagger}\OP{sim}(v_i,v_j)\\
&=\frac1{|\CAL E\sqcup(\CAL E^\dagger\setminus\CAL E)|}\sum_{(v_i,v_j)\in\CAL E\sqcup(\CAL E^\dagger\setminus\CAL E)}\OP{sim}(v_i,v_j)\\
&=\frac{1}{|\CAL A\sqcup\CAL B|}\sum_{z\in\CAL A\sqcup\CAL B}z\\
&>\frac1{|\CAL A|}\sum_{z\in\CAL A}z\\
&=\frac1{|\CAL E|}\sum_{(v_i,v_j)\in\CAL E}\OP{sim}(v_i,v_j)\\
&=\OP{hom}(\CAL G)
.}

\vspace{0.5em}\noindent\textbf{Number of edges.} Since there are $|\CAL V|$ nodes in total, then the total number of node pairs is $\binom{|\CAL V|}2$. Recall that $\CAL E^\dagger\setminus\CAL E$ is the set of added edges. It follows that
\AL{|\CAL E^\dagger|-|\CAL E|&=|\CAL E^\dagger\setminus\CAL E|\le\binom{|\CAL V|}2\\&=\frac{|\CAL V|(|\CAL V|-1)}2=O(|\CAL V|^2).}

\subsection{Proof of Observation~\ref{lem:2hop}}
Since $\OP{sim}(v_i,v_j)=\|\BM X[v_i,:]\land\BM X[v_j,:]\|_\infty>0$, then there exists a feature $k\in\CAL X$ such that $\BM X[v_i,k]\land\BM X[v_j,k]>0$. Since the feature matrix $\BM X$ is binary, then we must have
\AL{\BM X[v_i,k]=1,\qquad\BM X[v_j,k]=1.}
This implies that $(v_i,x_k)\in\CAL E^*$ and that $(v_j,x_k)\in\CAL E^*$. Hence, there exists a length-$2$ path $v_i\to x_k\to v_j$ connecting graph nodes $v_i$ and $v_j$. Therefore, $v_i$ and $v_j$ are two-hop neighbors of each other. 

\subsection{Proof of Theorem~\ref{thm:eff}}
\vspace{0.5em}\noindent\textbf{Homophily.} Since the original graph $\CAL G$ is homophilic, then there exists a pair of nodes, $v_i,v_j\in\CAL V$ ($v_i\ne v_j$), such that $\OP{sim}(v_i,v_j)=\|\BM X[v_i,:]\land\BM X[v_j,:]\|_\infty>0$ but $(v_i,v_j)\notin\CAL E$. Since $\OP{sim}(v_i,v_j)=\|\BM X[v_i,:]\land\BM X[v_j,:]\|_\infty>0$, then there exists a feature $k\in\CAL X$ such that $\BM X[v_i,k]\land\BM X[v_j,k]>0$. Since the feature matrix $\BM X$ is binary, then we must have
\AL{\BM X[v_i,k]=1,\qquad\BM X[v_j,k]=1.}
This implies that $(v_i,x_k)\in\CAL E^*\setminus\CAL E$ and that $(v_j,x_k)\in\CAL E^*\setminus\CAL E$. Thus, $\CAL E^*\setminus\CAL E$ is nonempty.

Furthermore, for any feature node $x_k\in\CAL V_\CAL X$, since any feature edge $(v_i,x_k)\in\CAL E_\CAL X$ ensures $\BM X[v_i,k]=1$, then we have
\AL{
\BM X^*[x_k,k]&=\frac1{|\CAL E_\CAL X\cap(\CAL V\times\{x_k\})|}\sum_{v_i:(v_i,x_k)\in\CAL E_\CAL X}\BM X[v_i,k]\\
&=\frac1{|\CAL E_\CAL X\cap(\CAL V\times\{x_k\})|}\sum_{v_i:(v_i,x_k)\in\CAL E}1\\
&=\frac1{|\CAL E_\CAL X\cap(\CAL V\times\{x_k\})|}\sum_{v_i:(v_i,x_k)\in\CAL E\cap(\CAL V\times\{x_k\})}1\\
&=1
.}

Finally, for any added feature edge $(v_i,x_k)\in\CAL E^*\setminus\CAL E=\CAL E_\CAL X$,
\AL{
\OP{sim}(v_i,x_k)&=\|\BM X[v_i,:]\land\BM X[x_k,:]\|_\infty\\
&=\max_{k'\in\CAL X}|\BM X[v_i,k']\land\BM X[x_k,k']|\\
&\ge|\BM X[v_i,k]\land\BM X[x_k,k]|\\
&=|1\land1|=1
.}
Since $\OP{hom}(\CAL G)<1$, then
\AL{\OP{sim}(v_i,x_k)\ge1>\OP{hom}(\CAL G).}
Therefore, by Lemma~\ref{lem:add} with
\AL{
\CAL A&:=\{\!\!\{\OP{sim}(v_i,v_j):(v_i,v_j)\in\CAL E\}\!\!\},\\
\CAL B&:=\{\!\!\{\OP{sim}(v_i,x_k):(v_i,x_k)\in\CAL E_\CAL X\}\!\!\},
}
we have
\AL{
\OP{hom}(\CAL G^*)&=\frac1{|\CAL E^*|}\sum_{(u,u')\in\CAL E^*}\OP{sim}(u,u')\\
&=\frac1{|\CAL E\sqcup\CAL E_\CAL X|}\sum_{(u,u')\in\CAL E\sqcup\CAL E_\CAL X}\OP{sim}(u,u')\\
&=\frac{1}{|\CAL A\sqcup\CAL B|}\sum_{z\in\CAL A\sqcup\CAL B}z\\
&>\frac1{|\CAL A|}\sum_{z\in\CAL A}z\\
&=\frac1{|\CAL E|}\sum_{(v_i,v_j)\in\CAL E}\OP{sim}(v_i,v_j)\\
&=\OP{hom}(\CAL G)
.}

\vspace{0.5em}\noindent\textbf{Number of nodes.} Since $|\CAL X|\le O(|\CAL V|)$, then
\AL{|\CAL V_\CAL X|=|\CAL X|\le O(\CAL V).}
It follows that
\AL{|\CAL V^*|&=|\CAL V|+|\CAL V_\CAL X|\\
&\le|\CAL V|+O(|\CAL V|)\\
&=O(|\CAL V|)
.}

\vspace{0.5em}\noindent\textbf{Number of edges.} Since $\BM X$ is a binary matrix, then $\|\BM X\|_1=\|\BM X\|_0\le O(|\CAL E|)$. Hence,
\AL{
|\CAL E_\CAL X|&=\sum_{v_i\in\CAL V}\sum_{x_k\in\CAL V_\CAL X}1_{[(v_i,x_k)\in\CAL E_\CAL X]}\\
&=\sum_{v_i\in\CAL V}\sum_{k\in\CAL X}1_{[(v_i,x_k)\in\CAL E_\CAL X]}\\
&=\sum_{v_i\in\CAL V}\sum_{k\in\CAL X}1_{[\BM X[v_i,k]=1]}\\
&=\sum_{v_i\in\CAL V}\sum_{k\in\CAL X}\BM X[v_i,k]\\
&=\sum_{v_i\in\CAL V}\sum_{k\in\CAL X}|\BM X[v_i,k]|\\
&=\|\BM X\|_1=\|\BM X\|_0\le O(|\CAL E|)
.}
It follows that
\AL{
|\CAL E^*|&=|\CAL E|+|\CAL E_\CAL X|\\
&\le|\CAL E|+O(|\CAL E|)\\
&=O(|\CAL E|)
.}

\section{Use of Large Language Models}

We made limited and controlled use of large language models (LLMs) solely for stylistic refinement and improving readability of the text. All scientific content, methodology, experiments, and conclusions were fully conceived and validated by the authors. The role of LLMs was purely editorial and does not constitute co-authorship.

\end{document}